\documentclass{article}
\usepackage{tikz,makecell,amsfonts,amsmath,amsthm,amssymb,booktabs,url}
\usepackage{subcaption}
\usepackage{authblk}
\newtheorem{thm}{Theorem}[section]
\newtheorem{cor}[thm]{Corollary}
\newtheorem{lem}[thm]{Lemma}

\newtheorem{dee}[thm]{Definition}

\newtheorem{rem}[thm]{Remark}

\title{Pricing Strategies for Different Accuracy Models from the Same Dataset Based on Generalized Hotelling's Law}

\author[1]{Jie Liu\footnote{Equal contribution}}
\author[2]{Tao Feng\footnote{Equal contribution}}
\author[2]{Peizheng Wang}
\author[2]{Chao Wu\footnote{Corresponding author}}
\affil[1]{Zhejiang Shuren University}
\affil[2]{Zhejiang University}
\date{}
\begin{document}

\begin{titlepage}

\maketitle
\begin{abstract}
We consider a scenario where a seller possesses a dataset $D$ and trains it into models of varying accuracies for sale in the market. Due to the reproducibility of data, the dataset can be reused to train models with different accuracies, and the training cost is independent of the sales volume. These two characteristics lead to fundamental differences between the data trading market and traditional trading markets. The introduction of different models into the market inevitably gives rise to competition. However, due to the varying accuracies of these models, traditional multi-oligopoly games are not applicable. We consider a generalized Hotelling's law, where the accuracy of the models is abstracted as distance. Buyers choose to purchase models based on a trade-off between accuracy and price, while sellers determine their pricing strategies based on the market's demand. We present two pricing strategies: static pricing strategy and dynamic pricing strategy, and we focus on the static pricing strategy. We propose static pricing mechanisms based on various market conditions and provide an example. Finally, we demonstrate that our pricing strategy remains robust in the context of incomplete information games.
\end{abstract}
\clearpage
\setcounter{tocdepth}{2} 
\tableofcontents

\end{titlepage}

\section{Introduction}
Data training and trading are receiving increasing attention\cite{agarwal2019marketplace,pei2020survey,zhang2023survey}. Industries across the board have focused their attention on how to maximize the extraction of information inherent in data through model training. The value of this information is precisely the manifestation of the value of data as a production factor itself.

We focuses on considering trained models as products in this paper. Trading models rather than data itself inherently reduces the risk of data privacy breaches. Unlike physical materials, data possesses reproducibility. In \cite{agarwal2019marketplace,chen2022selling}, the reproducibility of data is identified as a detrimental factor, and strategies are explored to address the challenge of diminished data value arising from this reproducibility. However, when data is utilized solely as raw material for training models rather than being directly sold in the market, we observe that due to this characteristic, we can retrain the data to generate different models and sell them. Also, from the market's perspective, different models can be trained using the same data set, thereby satisfying varying market demands. This results in traditional economic principles, such as opportunity cost, acquiring different connotations in the context of data transactions. 

Despite this, there are still many similarities between data trading and traditional market transactions. For instance, when different model products are sold in the market, competition will inevitably arise. There has been lots of research regarding game theory in data valuation and data market transactions, such as \cite{jia2019towards,karlavs2022data,huang2023evaluating,yu2023swdpm}. We focus on considering the application of generalized Hotelling's law in the data trading market.

    The traditional Hotelling's law is a spatial competition strategy, and it describes a scenario where two firms compete by setting prices and choosing locations along a linear product space\cite{brown1989retail}. Based on Hotelling's law, there have been numerous studies on location selection and pricing, see \cite{guo2015hotelling,10.1145/3292522.3326020}. {\bf Essentially, Hotelling's law reflects a trade-off between two independent variables of buyers based on their needs and preferences in order to make an optimal choice.} In this paper, we quantify model accuracy as the distance between multiple buyers and sellers. Buyers weigh the unit price and accuracy of different models based on their own utility functions and then choose the model that maximizes their utility.

We consider a scenario where a single buyer chooses between these two models at first. In Figure \ref{datahotelling},  the buyer naturally prefers a model with $100 \%$ accuracy, meaning the buyer’s location corresponds to the location of the $100 \%$ accurate model. However, such accuracy is impossible to achieve. Therefore, the sellers can only strive to maximize the model’s accuracy, i.e., to be “closer” to the buyer in terms of “location”. On the other hand, the buyer compares the utility value between $T_1$ and $T_2$ and chooses the model with higher utility.

Besides replicability, data itself also has privacy protection features, which often lead to incomplete information on both the buyer's and seller's sides. In \cite{shokri2012protecting}, the authors explore how to ensure the robustness of the results with incomplete information. In Section \ref{sec:incinf}, we also demonstrate that under incomplete information game, the algorithm we provide remains robust and the error bound can be calculated.

{\bf Related works} 

In \cite{agarwal2019marketplace}, the authors introduce an effective data trading mechanism, which gives the strategies for matching buyers and sellers as well as how the market interacts. Additionally, the authors present a robust-to-replication algorithm, ensuring that even if the traded data is replicated, it will not affect its market pricing. In \cite{chen2022selling}, the authors introduce another data trading mechanism to prevent the devaluation of the seller’s data after replication. This method involves the seller initially presenting a portion of the data to the buyer as sample data along with a selling price. Then the buyer conducts Bayesian inference to predict the accuracy and quality of the data via the corresponding prior knowledge, thus determining whether to purchase the data at the offered price or not. In terms of privacy protection, \cite{amiri2023fundamentals} presents a trading mechanism that not only protects data privacy but also allows the seller to demonstrate the quality of the data to the buyer. In this paper, the quality of the data depends on the relevance of the sold data to the buyer’s needs. In \cite{ravindranath2023data}, a deep learning-based approach for designing data market mechanisms is presented, targeting more complex transaction scenarios. It seeks to identify optimal market mechanism designs by examining two settings: one involving a single buyer and another involving multiple buyers. In \cite{10.1145/3677127}, a data market pricing scheme based on the combination of machine learning and the Stackelberg game is also proposed, and the optimization objective of this pricing scheme is to maximize the buyers' benefits. In \cite{egan2024package}, a scenario of repeatedly selling data in an auction setting is given, and the corresponding winner determination problem is investigated. Regarding research on direct model pricing, \cite{10.1145/3299869.3300078} studies a machine learning-based model pricing mechanism, with a particular focus on strategies to avoid arbitrage.

\section{Preliminary: Data chain structure and utility functions}\label{sec:modeling}
In this paper, we consider a scenario where sellers train models with different accuracies using the dataset $D$ and sell them in the market. We arrange these models as a data chain, see Figure ~\ref{fig3}, based on their accuracies, from lowest to highest (or vice versa). We assume in this paper that accuracy $A_i$ of the model $T_i$ is a strictly monotonically increasing known function of training cost $c_i$. We aim to maximize the profit of models across the entire data chain by arranging the training costs and unit prices $p_i$. Note that the cost of model training is independent of sales volume, and there exists competition among models of different accuracies. Let $\mathbb{S}:s_1\to s_2\to\cdots\to s_n$ be a data chain and $T_i$'s are the corresponding trained models, the goal is to maximize $\underset{i=1}{\stackrel{n}{\sum}}\max\{r_i-c_i,0\}$ under the constraints of $c_i$'s where $r_i$ is the revenue from sales of $T_i$.

\begin{figure}
    \centering
    \begin{tikzpicture}
    \draw[->] (0,0) -- (7,0);
    \node[anchor=east, yshift=5pt] at (9,0) {accuracy};
    \filldraw (2,0) circle (1pt) node[above=3pt] {$A_1$};
    \filldraw (5,0) circle (1pt) node[above=3pt] {$A_2$};
    \filldraw (2,0) circle (1pt) node[below=3pt] {$T_1$(with price $p_1$)};
    \filldraw (5,0) circle (1pt) node[below=3pt] {$T_2$(with price $p_2$)};
    \filldraw (7,0) circle (1pt) node[above=3pt]
    {$1$};
    \filldraw (7,0) circle (1pt) node[below=3pt] {buyer};
\end{tikzpicture}
    \caption{Generalized Hotelling's law}
    \label{datahotelling}
\end{figure}

\begin{figure}[h]
\centering
\includegraphics[scale=0.3]{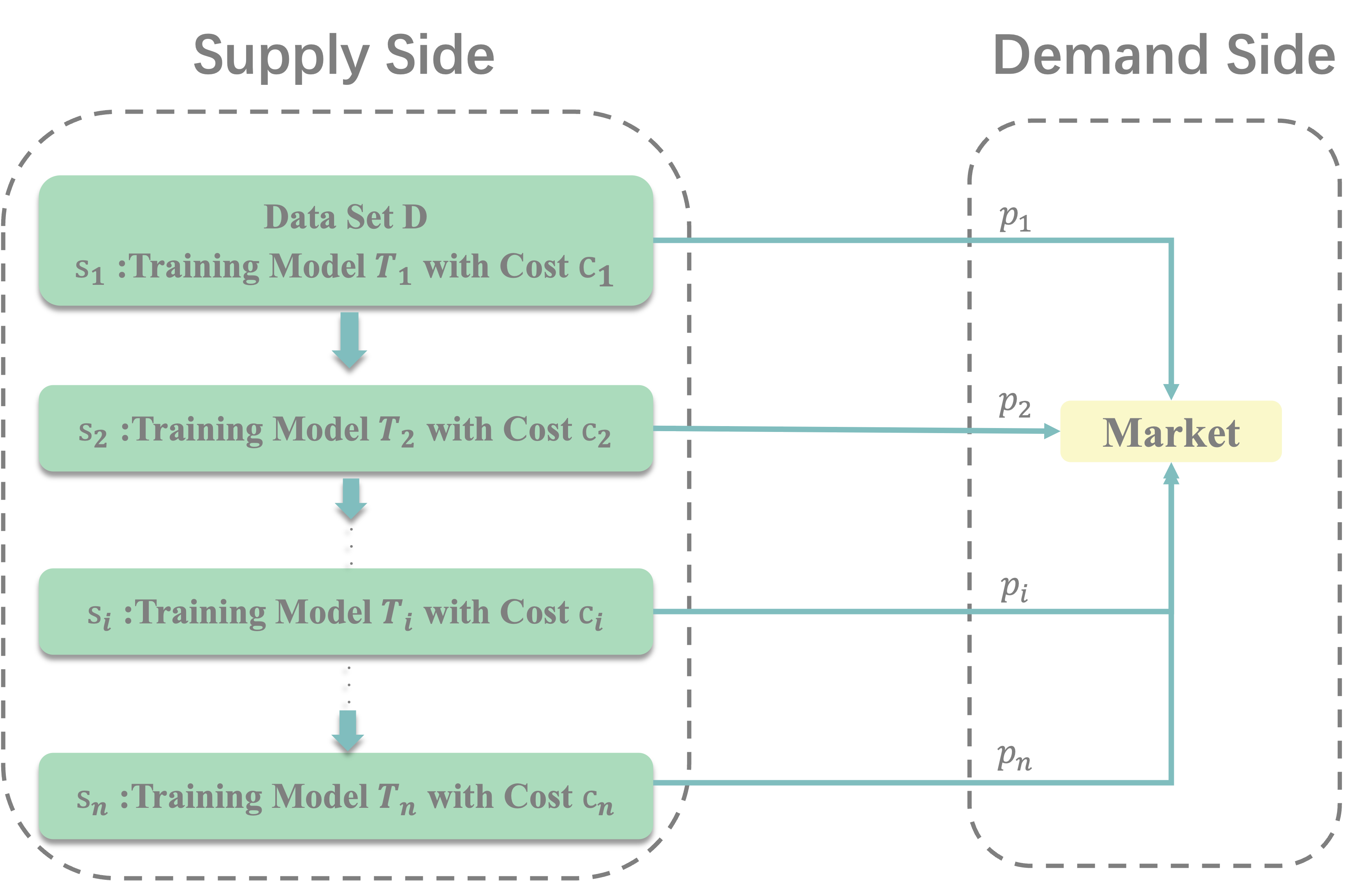}
\caption{Data Chain (Dataset D cannot be directly sold to the market) } 
\vspace{-10pt}
\label{fig3}
\end{figure}

An obvious issue in solving this optimization problem is how to determine the unit price and sales volume for each model. Generally, the seller's sales volume is related to market demand, expressed as utility functions once the unit price is set. 

We denote a class of continuous functions $\mathfrak{B}=(b_1,b_2,\cdots,b_n)$ as utility functions for all buyers, where each index corresponds to the model $T_i$ with cost $c_i$ and the precision $A_i$ determined by $c_i$. The independent variables of $b_i$'s include the model price $p$ and the buyers' minimum accuracy requirement $a$. The buyer will choose the model only if the utility function value is positive and the model's accuracy exceeds their own accuracy requirements. If multiple models have positive utility function values, the buyer will choose the model with the highest function value. We denote $P$ as the price cap and we define $\underline{A}$ and $\overline{A}$ as the lower and upper bounds of the model's accuracy, respectively. We assume that $b_i$'s have the following properties:
\begin{itemize}
\item For any two models $T_i$ and $T_j$ with $i<j$ and fixed $p$ and $a$, then $b_i(p,a)<b_j(p,a)$.
\item All $b_i$'s are strictly monotonically decreasing of $p$. 
\item The function $b_i$ is strictly monotonically increasing of $a\in[\underline{A},A_i]$.
\end{itemize}

The first property implies that, for buyers with a fixed accuracy requirement, a model with higher accuracy corresponds to higher utility at the same price point. The third property indicates that, when a model is priced fixedly, buyers with higher accuracy needs will derive higher utilities from the model.

We define $D_{i,p}=\{a\in[\underline{A},A_i]|b_i(p,a)\geq0\}$ and we have

\begin{dee}
    We name the class of functions $\mathfrak{B}=(b_1,b_2,\cdots,b_n)$ are accuracy-compatible if for any $i<j$ and fixed $p,p'$, we denote $B_{ij}(a)=b_i(p,a)-b_j(p',a)$ and we have
    \begin{itemize}
        \item $B_{ij}(a)>0$ or $B_{ij}(a)<0$ for $a\in D_{i,p}\bigcap D_{j,p'}$.
        \item If there exists $a\in D_{i,p}\bigcap D_{j,p'}$ such that $B_{ij}(a)=0$, then  $B_{ij}(a_+)<0$ for $D_{i,p}\bigcap D_{j,p'}\ni a_+>a$ and $B_{ij}(a_-)>0$ for $D_{i,p}\bigcap D_{j,p'}\ni a_-<a$.
    \end{itemize}
\end{dee}
\begin{rem}
    The motivation behind defining the accuracy-compatibility of utility functions stems from the observation that, when the prices of two models fall within a reasonable range, i.e., the corresponding utility function values are non-negative, and then these two models are in competition. In this case, the outcome of this competition consistently sees buyers with high accuracy demands opting for the model that offers high accuracy at a higher price. In contrast, buyers with low accuracy demands make the opposite choice. This definition aligns with realistic scenarios.
\end{rem}


\section{Static pricing strategy: generalized Hotelling's law}\label{sec:stathote}
We reiterate that in this paper, we consider a seller possesses a dataset $D$ and makes a profit by training this dataset into multiple models of different accuracies and selling them. We assume that, for the seller, the relationship between the training cost and the corresponding model accuracy is a known, strictly monotonically increasing function. Also, for the buyers, the distribution of their minimum required accuracies is known, which we denote as $\lambda(a)$, and continuous. Moreover, for any buyer with a known minimum required accuracy, the utility function for models of a specified accuracy and price is also known, and the utility functions corresponding to all accuracies are accuracy-compatible.

We give the algorithm steps to maximize the seller's profit,
\begin{enumerate}
    \item The seller lists all the possible cost allocations $c_1,c_2,\cdots,c_n$ where $c_i<c_{i+1}$.
    \item The seller optimizes the corresponding pricing strategy to maximize the profit based on the fixed cost allocations.
    \item The seller compares all the profits and chooses the cost allocation corresponding to the highest profit.
\end{enumerate}

We focus primarily on the second step in this section. We consider a static pricing strategy, where upstream nodes set prices first, and downstream nodes set prices based on the distribution of minimum accuracy requirements and the upstream nodes' pricing strategies to maximize the profit of this model and the models of the upstream nodes. Once downstream nodes have completed their pricing, upstream nodes do not adjust their prices further. We use generalized Hotelling's law to describe the market competition between adjacent nodes. Specifically, we have
\begin{enumerate}
    \item We calculate $p_1$ to maximize the profit of $T_1$.
    \item Fixed $p_1$ and we calculate $p_2$ to maximize the profit of $T_1$ and $T_2$.
    \item We iterate to calculate $p_i$, while keeping $p_1, p_2, \cdots, p_{i-1}$ invariant, until $i=n$.
\end{enumerate}

\subsection{Two nodes example}
When only model $T_1$ is sold in the market, the optimal pricing strategy which is based on the buyers’ utility function and minimum accuracy requirements is
\begin{equation}\label{i1case}
    p_1=\underset{p}{argmax}\quad p\int_{a'_1}^{A_1}\lambda(a)da.
\end{equation}
where $b_1(p,a'_1)=0$. Note that $A_1$ is the function of $c_1$ and utility function $b_1$ is denoted by $A_1$. We rename the market allocation of $T_1$ as $(a_{1,-},a_{1,+}]$.

We consider the situation in which both $T_1$ and $T_2$ are sold in the market, and the corresponding utility functions $b_1$ and $b_2$ are accuracy-compatible. Note that due to differences in model accuracy and price, their utility function values vary, leading to differing competitive environments when these models are sold into the market. Consequently, we categorize and discuss all possible competitive environments and present corresponding optimization plans along with the constraints on accuracy and price for models associated with these strategies. We denote $b_2(p,a'_2)=0$ and $b_1(p_1,a_2^{\ast})=b_2(p,a_2^{\ast})\geq0$ if $a_2^\ast$ is attainable. We classify and discuss three possible market allocation scenarios:

\begin{itemize}
    \item If the $T_2$'s market completely covers the $T_1$'s market, then we have
    \begin{equation}\label{i2case1}
        p_2=\underset{p}{argmax}\,p\int_{a'_2}^{A_2}\lambda(a)da,\quad
        a'_2< a_{1,-}.
\end{equation}
Note that $B_{12}<0$ in $(a_{1,-},a_{1,+}]$ because of accuracy-compatible of $b_i$'s. It leads to $a_2^\ast$ not attainable.
    \item If there is market competition between $T_2$ and $T_1$, then we have
    \begin{equation}\label{i2case2}
        p_2=\underset{p}{argmax}\,p_1\int_{a_{1,-}}^{a_2^{\ast}}\lambda(a)da+p\int_{a_2^{\ast}}^{A_2}\lambda(a)da,\quad a'_2\geq a_{1,-}\quad 
        a_{1,-}\leq a_2^\ast\leq a_{1,+}.
\end{equation}
    \item If there is no market competition between $T_1$ and $T_2$, then we have
    \begin{equation}\label{i2case3}
        p_2=\underset{p}{argmax}\,p\int_{A_1}^{A_2}\lambda(a)da,\quad
        a_{1,-}<a'_2\leq A_1,
\end{equation}
or
\begin{equation}\label{i2case4}
        p_2=\underset{p}{argmax}\,p\int_{a'_2}^{A_2}\lambda(a)da,\quad
        A_1<a'_2<A_2.
\end{equation}
\end{itemize}

It is clear that the market allocation of each model is connected, and we rename the new market allocation of 
$T_1$ and $T_2$ as $(a_{1,-},a_{1,+}]$ and $(a_{2,-},a_{2,+}]$. Let
\begin{equation}
        \mathfrak{b}_2(a)=\begin{cases}
b_j(p_j,a) & \text{if}\quad a\in(a_{j,-},a_{j,+}],\quad j=1,2, \\
0 & \text{otherwise,}
\end{cases}
    \end{equation}
we show the market allocation that corresponds to these situations, see Figure \ref{fig:i2case}.

\begin{figure}[htbp]
    \centering
    \begin{minipage}{0.48\textwidth}
        \centering
        \begin{subfigure}[b]{0.48\textwidth}
            \centering
            \begin{tikzpicture}[scale=0.75]
                \draw[->] (0,0) -- (4,0) node[right] {$a$};
                \draw[->] (0,-1) -- (0,2) node[above] {$\mathfrak{b}_2(a)$}; 
                \coordinate (a) at (0.5,0);
                \coordinate (b) at (1.5,0);
                \coordinate (c) at (2.5,0);
                \coordinate (d) at (3.5,0);
                \node[below] at (a) {$a_{2,-}$};
                \node[below] at (d) {$a_{2,+}$};
                \draw[thick] (a) to (3.5,2); 
                \fill[blue, opacity=0.3] (0.5,-0.1) rectangle (3.5,0.1);
            \end{tikzpicture}
            \caption{Allocation of \eqref{i2case1}}
            \label{fig:mono1}
        \end{subfigure}
        \hspace{0.04\textwidth}
        \begin{subfigure}[b]{0.48\textwidth}
            \centering
            \begin{tikzpicture}[scale=0.75]
                \draw[->] (0,0) -- (4,0) node[right] {$a$};
                \draw[->] (0,-1) -- (0,2) node[above] {$\mathfrak{b}_2(a)$};
                \coordinate (a) at (0.5,0);
                \coordinate (b) at (1.5,0);
                \coordinate (c) at (3.5,0);
                \node[below] at (a) {$a_{1,-}$};
                \node[below] at (b) {$a_{1,+}$($a_{2,-}$)};
                \node[below] at (c) {$a_{2,+}$};
                \draw[thick] (a) to (1.5,0.5);
                \draw[thick] (1.5,0.5) to (3.5,2.5);
                \fill[red, opacity=0.3] (0.5,-0.1) rectangle (1.5,0.1);
                \fill[blue, opacity=0.3] (1.5,-0.1) rectangle (3.5,0.1);
            \end{tikzpicture}
            \caption{Allocation of \eqref{i2case2}}
            \label{fig:mono2}
        \end{subfigure}
    \end{minipage}
    \begin{minipage}{0.48\textwidth}
        \centering
        \begin{subfigure}[b]{0.48\textwidth}
            \centering
            \begin{tikzpicture}[scale=0.75]
                \draw[->] (0,0) -- (4,0) node[right] {$a$};
                \draw[->] (0,-1) -- (0,2) node[above] {$\mathfrak{b}_2(a)$};
                \coordinate (a) at (0.5,0);
                \coordinate (b) at (1.5,0);
                \coordinate (c) at (3.5,0);
                \node[below] at (a) {$a_{1,-}$};
                \node[below] at (b) {$a_{1,+}$($a_{2,-}$)};
                \node[below] at (c) {$a_{2,+}$};
                \draw[thick] (a) to (1.5,1); 
                \draw[thick] (1.5,0.5) to (3.5,3);
                \fill[red, opacity=0.3] (0.5,-0.1) rectangle (1.5,0.1);
                \fill[blue, opacity=0.3] (1.5,-0.1) rectangle (3.5,0.1);
            \end{tikzpicture}
            \caption{Allocation of \eqref{i2case3}}
            \label{fig:mono3}
        \end{subfigure}
        \hspace{0.04\textwidth}
        \begin{subfigure}[b]{0.48\textwidth}
            \centering
            \begin{tikzpicture}[scale=0.75]
                \draw[->] (0,0) -- (4,0) node[right] {$a$};
                \draw[->] (0,-1) -- (0,2) node[above] {$\mathfrak{b}_2(a)$};
                \coordinate (a) at (0.5,0);
                \coordinate (b) at (1.5,0);
                \node[below] at (a) {$a_{1,-}$};
                \node[below] at (b) {$a_{1,+}$};
                \draw[thick] (a) to (1.5,1); 
                \draw[thick] (2,0) to (3.5,3);
                \fill[red, opacity=0.3] (0.5,-0.1) rectangle (1.5,0.1);
                \fill[blue, opacity=0.3] (2,-0.1) rectangle (3.5,0.1);
                \coordinate (c) at (2,0);
                \node[below] at (c) {$a_{2,-}$};
                \coordinate (d) at (3.5,0);
                \node[below] at (d) {$a_{2,+}$};
            \end{tikzpicture}
            \caption{Allocation of \eqref{i2case4}}
            \label{fig:mono4}
        \end{subfigure}
    \end{minipage}
    \caption{Market allocations(red is for $T_1$ and blue is for $T_2$)}
    \label{fig:i2case}
\end{figure}

\subsection{General situation}\label{subsec:stagen}
In this subsection, we prove the main result by induction. For a data chain length of $2$, we have the following observations
\begin{itemize}
    \item The market allocation of each model is connected and $\mathfrak{b}_2$ is continuous almost everywhere. Moreover, we have $a_{1,-}<a_{1,+}\leq a_{2,-}<a_{2,+}$.
    \item If $\mathfrak{b}_2$ is not continuous in $\mathfrak{a}$, then $\underset{a\to\mathfrak{a}^-}{\lim}\mathfrak{b}_2(a)>\underset{a\to\mathfrak{a}^+}{\lim}\mathfrak{b}_2(a)$.
    \item If there exists market allocation $(a_{j,-},a_{j,+}]$ corresponds to $T_j$ and $a_{j,-}\notin supp\{\mathfrak{b}_2\}$, then $\underset{a\to (a_{j,-})^+}{\lim} b_j(a)=0$.
    \item If $\mathfrak{b}_i$ is discontinuous at point $a_{k,+}$, then $a_{k,+}=A_k$. 

\end{itemize}
We then assume that it also holds when the length of the data chain is $i$, and we introduce a definition
\begin{dee}
We define $i$-th enveloping utility function of the data chain $s_1\to s_2\to\cdots\to s_i$ as
    \begin{equation}
        \mathfrak{b}_i(a)=\begin{cases}
b_j(p_j,a) & \text{if}\quad a\in(a_{j,-},a_{j,+}],\quad j=1,2,\ldots,i, \\
0 & \text{otherwise.}
\end{cases}
    \end{equation}
\end{dee}

When the length of data chain is  $i+1$, we have
\begin{lem}
    Fixed $p$, $b_{i+1}(p,a)$ and $\mathfrak{b}_i(a)$ in $supp\{\mathfrak{b}_i\}$ have at most one point where their function values are equal.
\end{lem}
\begin{proof}
    We prove this lemma by contradiction. We assume there are two adjacent points where $b_j(p_j,\alpha_1)=b_{i+1}(p,\alpha_1)$ and $b_k(p_k,\alpha_2)=b_{i+1}(p,\alpha_2)$. It is clear that $j\neq k$ due to the accuracy-compatibility. Without losing generality, we assume $j<k$, and then $\alpha_1<\alpha_2$ by definition.

    We have $b_{i+1}(p,a_{j,+})\geq\mathfrak{b}_i(a_{j,+})=b_j(p_j,a_{j,+})$, and $b_{i+1}(p,a_{k,-})<\underset{a\to (a_{k,-})^+}{\lim}\mathfrak{b}_i(a_{k,-})=b_k(p_k,a_{k,-})$ respectively, due to the accuracy-compatibility of $b_{i+1}$, $b_j$ and $b_k$. Also, we have $b_{i+1}(p,a_{j,+})\leq b_{i+1}(p,a_{k,-})$ due to monotonicity. Note that $b_{i+1}(p,a_{j,+})=b_{i+1}(p,a_{k,-})$ if and only if $j+1=k$, i.e., $a_{j,+}=a_{k,-}$, but in this case, we have $b_{i+1}(p,a_{j,+})\geq b_{j}(p_j,a_{j,+})\geq b_k(p_k,a_{k,-})>b_{i+1}(p,a_{k,-})$, it leads to a contradiction. So we claim that $b_{i+1}(p,a_{j,+})<b_{i+1}(p,a_{k,-})$ and $a_{j,+}\neq a_{k,-}$.

    We have
    \begin{itemize}
        \item If $(a_{j,+},a_{k,-}]\subseteq supp\{\mathfrak{b}_i\}$, then we denote the market allocations of $T_{j+1},T_{j+2},\cdots,T_{k-1}$ are $(a_{s,-},a_{s,+}]$($s=j+1,j+2,\ldots,k-1$). Due to compatibility and the fact that $b_{i+1}(p,a_{j,+})\geq b_j(p_j,a_{j,+})$, we have, note that $a_{s,+}=a_{s+1,-}$,
        \begin{equation}
            b_{i+1}(p,a)-b_s(p_s,a)>0,\quad a\in (a_{s,-},a_{s,+}],\quad s=j+1,j+2,\ldots,k-1.
        \end{equation}
        It leads to $b_{i+1}(p,a_{k,-})\geq \underset{a\to (a_{k,-})^+}{\lim}b_k(p_k,a_{k,-})>b_{i+1}(p,a_{k,-})$, which is a contradiction.
        \item If $(a_{j,+},a_{k,-}]\nsubseteq supp\{\mathfrak{b}_i\}$, then there exists $(a_{u,-},a_{k,-}]\subseteq supp\{\mathfrak{b}_i\}$ and by assumption, we have $\underset{a\to (a_{u,-})^+}{\lim}\mathfrak{b}_i(a)=0$. So
        \begin{equation}
            b_{i+1}(p,a)-b_s(p_s,a)>0,\quad a\in (a_{s,-},a_{s,+}],\quad s=u+1,u+2,\ldots,k-1,
        \end{equation}
        which is also a contradiction.
    \end{itemize}
\end{proof}

We assume that for the data chain $s_1\to s_2\to \cdots\to s_i$, the corresponding market allocation for each model is connected. Then, for the model $T_{i+1}$ and let $b_{i+1}(p,a'_{i+1})=0$ and $b_{i+1}(p,a_{i+1}^\ast)=\mathfrak{b}_i(a_{i+1}^\ast)\geq0$ if $a_{i+1}^\ast$ is attainable, we consider the following possible market allocation scenarios:
\begin{itemize}
    \item If $a'_{i+1}\in(a_{j-1,+},a_{j,-}]$, then $b_{i+1}(p,a)\geq\mathfrak{b}_i(a)$ where $a\in(a'_{i+1},a_{i,+}]$ due to accuracy-compatibility. We have
    \begin{equation}\label{geneicase1}
        p_{i+1}=\underset{p}{argmax}\,p\int_{a'_{i+1}}^{A_{i+1}}\lambda(a)da,\quad
        a_{j-1,+}< a'_{i+1}\leq a_{j,-}.
\end{equation}
Also, if $a'_{i+1}\in(a_{k,-},a_{k,+}]$ where $b_{i+1}(p,a)<\mathfrak{b}_i(a)$ for $a\in(a'_{i+1},a_{j,+}]$ and $b_{i+1}(p,a)>\mathfrak{b}_i(a)$ for $a\in(a_{j+1,-},A_i]$, then we have
    \begin{equation}\label{aiai+1case}
        p_{i+1}=\underset{p}{argmax}\,p\int_{a_{j,+}}^{A_{i+1}}\lambda(a)da,\quad
        a_{k,-}< a'_{i+1}\leq a_{k,+}
\end{equation}
and $a_{j,+}=A_j$. In these two cases, $T_{i+1}$ completely covers $T_j, T_{j+1},\cdots T_i$'s markets and has no competition with the previous $j-1$ models.
    \item If $a'_{i+1}\in(a_{k,-},a_{k,+}]$ and  exists $a_{i+1}^\ast\in(a_{j,-},a_{j,+}]$ such that $b_{i+1}(p,a_{i+1}^\ast)=\mathfrak{b}_i(a_{i+1}^\ast)\geq0$, then we have
    \begin{equation}\label{geneicase2}
        p_{i+1}=\underset{p}{argmax}\,p_j\int_{a_{j,-}}^{a_{i+1}^\ast}\lambda(a)da+p\int_{a_{i+1}^\ast}^{A_{i+1}}\lambda(a)da,\quad
        a_{k,-}< a'_{i+1}\leq a_{k,+},\quad a_{j,-}<a_{i+1}^\ast\leq a_{j,+}.
\end{equation}
In this case, $T_{i+1}$ completely covers the $T_{j+1}, T_{j+2},\cdots, T_i$'s markets but competes with the $j$-th model. 
    \item If $a'_{i+1}\in(A_i,A_{i+1}]$, then we have
\begin{equation}\label{geneicase4}
        p_i=\underset{p}{argmax}\,p\int_{a'_{i+1}}^{A_{i+1}}\lambda(a)da,\quad
        A_{i}< a'_{i+1}\leq A_{i+1}.
\end{equation}
In this case,  $T_i$ has no market competition with all previous models.
\end{itemize}

It is clear that all classifications related to the allocation of $T_{i+1}$ are connected. We select the optimal strategy with the highest profit and update the enveloping utility function accordingly.

We have
\begin{thm}\label{thm:connect}
    For any given cost allocation, the market allocations for all models are connected if the utility functions are accuracy-compatible.
\end{thm}
Based on observations, we present the following corollaries without proof.
\begin{cor}\label{cor1}
    In the case of $b_i(p,a'_i)=0$, $a'_i$ is a strictly monotonically increasing function of $p$.
\end{cor}
\begin{cor}
    For the case of \eqref{aiai+1case}, we have $b_{i+1}(p_{i+1},A_j)=0$.
\end{cor}

\subsection{Continuity of the revenue function}
The goal of this subsection is to prove that $r_{i+1}(p)$ is a continuous function of $p$ under certain conditions. We have $(a_{q.-},a_{q,+}]$($q=1,2,\ldots,i$) as the optimal market allocation for $T_1,T_2,\cdots,T_i$. We then introduce the model $T_{i+1}$ to the market. If there exists $(a_{k,-},a_{j,+}]\in supp\{\mathfrak{b}_i\}$ where $a_{k,-}\notin supp\{\mathfrak{b}_i\}$ and $\mathfrak{b}_i$ is discontinuous at point $a_{j,+}=A_j$, then we denote $(a_{k,-},A_j]$ as a {\bf block}. The market allocation consists of a few blocks.
\begin{lem}
     We let $b_{i+1}(\rho,\alpha)=0$ where $b_{i+1}(\rho,A_j)=b_j(p_j,A_j)$. For a price $p'>0$, the equation $b_{i+1}(p',a)=\mathfrak{b}_i(a)>0$ can hold if $a'\in(a_{k,-},\alpha]$ where $b_{i+1}(p',a')=0$ and $b_{i+1}(p',a)=\mathfrak{b}_i(a)$ cannot hold if $a'\in(\alpha,A_j]$.
\end{lem}
\begin{proof}
    For $a'\in(a_{k,-},\alpha]$ we denote $0=b_{i+1}(p',a')<\mathbf{b}_i(a')$. Note that $p'<\rho$ by Corollary \ref{cor1} and we have
    \begin{equation}
        b_{i+1}(p',A_j)>b_{i+1}(\rho,A_j)=b_j(p_j,A_j)=\mathbf{b}_i({A_j}).
    \end{equation}
Since $b_{i+1}(p',a)$ and $\mathbf{b}_i(a)$ are both continuous in $(a_{k,-},A_j]$, it leads to $b_{i+1}(p',a)=\mathbf{b}_i(a)>0$ is attainable. The situation of $a'\in(\alpha,A_j]$ is similar and we omit the details.
\end{proof}
\begin{thm}
     For a block $(a_{k,-},A_j]$, if $a_{i+1}^\ast$ is a continuous function of price $p$ and
     $\underset{a'_{i+1}\to (a_{k,-})^+}{\lim} b_i(p,a_{i+1}^\ast)=0$ where $b_{i+1}(p,a'_{i+1})=0$, then revenue function $r_{i+1}$ of $T_{i+1}$ is a continuous function of $a'_{i+1}$ in this block. Moreover, if it holds for all blocks, then $r_{i+1}$ is continuous in $[\underline{A},A_{i+1}]$ and it leads to $r_i$ is a continuous function of the price.
\end{thm}
\begin{proof}
    We list all the endpoints in Subsection \ref{subsec:stagen} and show that all of them are continuous.
    \begin{itemize}
        \item For $a'_{i+1}=\alpha$, we have $r_{i+1}(\alpha)=p\int_{A_j}^{A_{i+1}}\lambda(a)da$ and $r_{i+1}(a'_{i+1})\to p\int_{A_j}^{A_{i+1}}\lambda(a)da$ for $a'_{i+1}\to(\alpha)^-$. For $a'_{i+1}\to(\alpha)^+$, we have, note that $b_{i+1}(p,a)<\mathfrak{b}_i(a)$ where $a\in(a'_{i+1},A_j]$, it also leads to $r_{i+1}(a'_{i+1})=p\int_{A_j}^{A_{i+1}}\lambda(a)da$.
        \item For $a'_{i+1}=a_{k,-}$, we have $r_{i+1}(a'_{i+1})=p\int_{a'_{i+1}}^{A_{i+1}}\lambda(a) da$ and $r_{i+1}(a'_{i+1})\to p\int_{a'_{i+1}}^{A_{i+1}}\lambda(a) da$ if $a'_{i+1}\to(a_{k,-})^-$. For $a'_{i+1}\to(a_{k,-})^+$, if $a_{i+1}^\ast$ is attainable, then by assumption, we have $a_{i+1}^\ast\to a'_{i+1}$ and $r_{i+1}(a'_{i+1})\to p\int_{a'_{i+1}}^{A_{i+1}}\lambda(a) da$. The conclusion is consistent in the case of $a_{i+1}^\ast$ not attainable.
        \item For $a'_{i+1}=a_{j,+}\neq a_{j+1,-}$, we have $r_{i+1}(a'_{i+1})=p\int_{a'_{i+1}}^{A_{i+1}}\lambda(a)da$. For $a'_{i+1}\to (a_{j,+})^+$, it leads to  $r_{i+1}(a'_{i+1})\to p\int_{a'_{i+1}}^{A_{i+1}}\lambda(a)da$ by \eqref{geneicase1}. For $a'_{i+1}\to (a_{j,+})^-$, if $a_{i+1}^\ast$ is attainable, then $a'_{i+1}\leq a_{i+1}^\ast<a_{j,+}$ and $r_{i+1}(a'_{i+1})\to p\int_{a'_{i+1}}^{A_{i+1}}\lambda(a)da$. The conclusion is consistent in the case of $a_{i+1}^\ast$ not attainable.
        \item For $a'_{i+1}=A_i$, note that it is an endpoint of a block and $a_{i,+}=A_i$. The left-continuity follows. For $a'_{i+1}\to (A_i)^+$, we have $r_{i+1}(a'_{i+1})=p\int_{a'_{i+1}}^{A_{i+1}}\lambda(a)da\to p\int_{A_i}^{A_{i+1}}\lambda(a)da$.
    \end{itemize}

\end{proof}

\section{Dual and quasi-dual situation}
We consider the dual and quasi-dual situation in this section, where the maximum permissible price for each buyer in the market is known, and our goal is also to maximize the seller's profit. We assume that the maximum permissible price for buyers in the market follows a continuous distribution denoted as $\mu(p)$. In this scenario, we continue to use dual generalized Hotelling's law. However, in this part, we abstract the unit price of the model as distance and mark the precision of the model as characteristics on points, see Figure \ref{dualdatahotelling}. We aim to find a static pricing strategy and the corresponding market allocation that maximizes the seller's profit. Note that with a slight abuse of notation, we use $\mathfrak{B}=(b_1,b_2,\cdots,b_n)$ in all situations.

\subsection{Dual situation}
We denote a class of functions $\mathfrak{B}=(b_1,b_2,\cdots,b_n)$ as dual utility functions, where each index corresponds to the price $P_i$ allocated to model $T_i$. Similarly, the independent variables of $b_i$'s include the model accuracy $a$ and the buyers' maximum permissible price $p$. Properties are similar to the primal situation.
We denote $D_{i,a}=\{p\in[0,P]|b_i(a,p)\geq0\}$ and denote $P_1<P_2<\cdots<P_n$.

\begin{dee}
    We name the class of functions $\mathfrak{B}=(b_1,b_2,\cdots,b_n)$ are price-compatible if for any $i<j$ and fixed $a,a'\in[\underline{A},\overline{A}]$, we denote $B_{ij}(p)=b_i(a,p)-b_j(a',p)$ and we have
    \begin{itemize}
        \item $B_{ij}(p)>0$ or $B_{ij}(p)<0$ in $D_{i,a}\bigcap D_{j,a'}$.
        \item If there exits $p$ such that $B_{ij}(p)=0$, then $B_{ij}(p_-)<0$ for $D_{i,a}\bigcap D_{j,a'}\ni p_->p$ and $B_{ij}(p_+)>0$ for $D_{i,a}\bigcap D_{j,a'}\ni p_+<p$.
    \end{itemize}
\end{dee}

Note that in Section \ref{sec:stathote}, we have presented an algorithm for pricing strategy. In this section, we interchange two variables from the previous section and, therefore, omit the detailed algorithm for brevity.

\subsection{Quasi-dual situation}
We observe that due to duality, the assumption in the dual scenario also shifts to the seller listing the pricing menu of all the models, rather than the cost allocations. However, when we continue to employ the assumption from Section \ref{sec:stathote}, namely that the seller can select a reasonable cost allocation to maximize profit, the optimization mechanism in the dual scenario becomes inapplicable. Therefore, we consider the quasi-dual scenario, where we reiterate the premises and assumptions within this context.

We assume a class of functions $\mathfrak{B}=(b_1,b_2,\cdots,b_n)$ called quasi-dual utility functions, where each index corresponds to the accuracy $A_i$ allocated to the model $T_i$. Similarly, the independent variables of $b_i$'s include the model price $p$ and the buyers' maximum permissible price $\bar{p}$, and $p\leq\bar{p}$. We define $D_{i,p}=\{\bar{p}\in[0,P]|b_i(p,\bar{p})\geq0\}$. We assume $c_1<c_2<\cdots<c_n$.

\begin{dee}
    We name the class of functions $\mathfrak{B}=(b_1,b_2,\cdots,b_n)$ are second-type accuracy-compatible if for any $i<j$ and fixed $p,p'\in[0,P]$, we denote $B_{ij}(\bar{p})=b_i(p,\bar{p})-b_j(p',\bar{p})$ and we have
    \begin{itemize}
        \item $B_{ij}(\bar{p})>0$ or $B_{ij}(\bar{p})<0$ in $D_{i,p_i}\bigcap D_{j,p_j}$.
        \item If there exits $\bar{p}$ such that $B_{ij}(\bar{p})=0$, then $B_{ij}(\bar{p}_-)<0$ for $D_{i,p_i}\bigcap D_{j,p_j}\ni \bar{p}_->\bar{p}$ and $B_{ij}(\bar{p}_+)>0$ for $D_{i,p_i}\bigcap D_{j,p_j}\ni \bar{p}_+<\bar{p}$.
    \end{itemize}
\end{dee}

If there is one model sold in the market with price $p$, the corresponding cost $c_1$ and $b_1(p,\bar{p}_1)=0$, we have
\begin{equation}
    p_1=\underset{p}{argmax}\,p\int_p^{\bar{p}_1}\mu(\bar{p})d\bar{p}.
\end{equation}
We rename the market allocation of $T_1$ as $[\bar{p}_{1,-},\bar{p}_{1,+})$.

    We assume two models $T_1$ and $T_2$ are sold in the market, and the corresponding utility functions $b_1$ and $b_2$ are second-type accuracy-compatible. Let $b_2(p,\bar{p}_2)=0$ and $b_2(p,\bar{p}_2^\ast)=b_1(p_1,\bar{p}_2^\ast)>0$ if $\bar{p}_2^\ast$ is attainable, we have
    \begin{itemize}
    \item If the $T_2$'s market completely covers the $T_1$'s, then we have
    \begin{equation}
        p_2=\underset{p}{argmax}\,p\int_p^{\bar{p}_2}\mu(\bar{p})d\bar{p},\quad
        \bar{p}_2\geq\bar{p}_{1,+},\quad p\leq \bar{p}_{1,-}.
\end{equation}                                                           
    \item If there is market competition between $T_2$ and $T_1$, then we have
    \begin{equation}
        p_2=\underset{p}{argmax}\,p_1\int_{\bar{p}_{1,-}}^{\bar{p}_2^\ast}\mu(\bar{p})d\bar{p}+p\int^{\bar{p}_2}_{\bar{p}_2^\ast}\mu(\bar{p})d\bar{p},\quad
        \bar{p}_{1,+}\geq \bar{p}_2^\ast\geq \bar{p}_{1,-}.
\end{equation}
    \item If there is no market competition between $T_1$ and $T_2$, meaning that the buyers choosing these two models do not overlap, then we have
    \begin{equation}
        p_2=\underset{p}{argmax}\,p\int_{p}^{\bar{p}_2}\mu(\bar{p})d\bar{p},\quad
        p>\bar{p}_{1,+}.
\end{equation}
or
\begin{equation}
        p_2=\underset{p}{argmax}\,p\int_{\bar{p}_{1,+}}^{\bar{p}_2}\mu(\bar{p})d\bar{p},\quad
        \bar{p}_{1,-}<p\leq\bar{p}_{1,+}.
\end{equation}
\end{itemize}
We rename the market allocation of $T_1$ and $T_2$ as $[\bar{p}_{1,-},\bar{p}_{1,+})$ and $[\bar{p}_{2,-},\bar{p}_{2,+})$.
\begin{rem}
    We give a brief explanation to show that the other situations are impossible. We have
    \begin{itemize}
        \item If $\bar{p}_{2,-}<\bar{p}_{2,+}\leq \bar{p}_{1,-}<\bar{p}_{1,+}$, then we have $b_2(p_2,\bar{p}_{1,-})>b_1(p_1,\bar{p}_{1,-})$ which contradicts the fact that $\bar{p}_{1,-}$ is in the market allocation of $T_1$.
        \item If the market allocation of $T_1$ is not connected, then there exists $p'_1$, $p'_2$ and $p''_1$ where $0<b_1(p_1,p'_1)<b_2(p_2,p'_2)$ and $b_2(p_2,p'_2)<b_1(p_1,p''_1)$. It leads to a contradiction that $b_1$ and $b_2$ are second-type accuracy-compatible. The situation of $T_2$ is similar. 
    \end{itemize}
\end{rem}

\begin{dee}
We define $i$-th enveloping quasi-dual utility function of the data chain $s_1\to s_2\to\cdots\to s_i$ as
    \begin{equation}
        \mathfrak{b}_i(\bar{p})=\begin{cases}
b_j(p_j,\bar{p}) & \text{if}\quad \bar{p}\in[\bar{p}_{j,-},\bar{p}_{j,+}), \\
0 & \text{otherwise.}
\end{cases}
    \end{equation}
\end{dee}

The situation with arbitrary chain lengths can be generalized from the primal situation and the quasi-dual situation with chain length of $2$ combined. We omit the details of this part.

\begin{figure}
    \centering
    \begin{tikzpicture}
    \draw[->] (0,0) -- (10,0);
    \node[anchor=east, yshift=5pt] at (12,0) {price};
    \filldraw (2,0) circle (1pt) node[above=3pt] {$p_2$};
    \filldraw (7,0) circle (1pt) node[above=3pt] {$p_1$};
    \filldraw (2,0) circle (1pt) node[below=3pt] {$T_2$(with Accuracy $A_2$)};
    \filldraw (7,0) circle (1pt) node[below=3pt] {$T_1$(with Accuracy $A_1$)};
    \filldraw (10,0) circle (1pt) node[above=3pt]
    {$0$};
    \filldraw (10,0) circle (1pt) node[below=3pt] {buyer};
\end{tikzpicture}
    \caption{Dual generalized Hotelling's law}
    \label{dualdatahotelling}
\end{figure}

\subsection{Discussion on the data chains among four dimensions}

In this section, we have conducted various forms of duality transformations on the primal situation, specifically by interchanging the constraint conditions associated with the two variables of accuracy and price to different extents. This includes fully swapping all constraint conditions between these two variables, as well as partially swapping some conditions while keeping others unchanged, thereby deriving several distinct 'dual' forms, see Table \ref{tab:4situ}. Note that the fourth scenario is not described in this paper. However, its properties and the method for constructing its data chain can be derived analogously.

\begin{table}%
	\caption{Situation among four dimensions}
	\label{tab:4situ}
	\begin{minipage}{\columnwidth}
		\begin{center}
			\begin{tabular}{lllll}
				\toprule
				Situations   &  Optimized object & \makecell{Known market\\ distribution} & \makecell{Index of\\ utility functions} & \makecell{Variables of\\ utility functions}\\
				\toprule
                Primal   &  Accuracy & Accuracy & Accuracy & (Price,Accuracy)\\
                Dual   &  Price & Price & Price & (Accuracy,Price)\\
                Quasi-dual   &  Accuracy & Price & Accuracy & (Price,Price)\\
                Ultra-dual   &  Price & Accuracy & Price & (Accuracy,Accuracy)\\
				\bottomrule
			\end{tabular}
		\end{center}
		\bigskip\centering
	\end{minipage}
\end{table}%

\section{Dynamic pricing strategy: chain-like Nash equilibrium}
In the previous sections, we consider the static strategy where
\begin{itemize}
    \item  The goal of pricing $T_i$ is to maximize the overall profit of $T_1,T_2,\cdots,T_i$.
    \item  Once $T_i$ has completed the pricing, $T_1,T_2,\cdots, T_{i-1}$ do not adjust their prices further. 
\end{itemize}
In this section, we consider a different scenario named dynamic pricing strategy and
\begin{itemize}
    \item The goal of pricing $T_i$ is to maximize the profit of $T_i$.
    \item The pricing strategy for $T_i$ is based on its adjacent models, $T_{i-1}$ and $T_{i+1}$.
\end{itemize}
When the chain length is $1$, the scenario is equivalent to a static pricing strategy. For the case where the chain length is $2$, we use the idea of bilateral game theory in data markets from \cite{d2012game}, which states that a Nash equilibrium exists when the maximum points of the corresponding best response functions for both players coincide. We consider the scenario where the chain length is $2$ and subsequently generalize it. Note that we incorporate an additional assumption, namely, that each model will always be trained and sold regardless of whether the profit is positive.

\subsection{Two nodes example}
When considering the scenario with two models $T_1$ and $T_2$, we note that whether competition arises between the two models depends on the choice of $T_2$. We have
\begin{itemize}
    \item {\bf In $T_1$'s perspective} If $T_2$ competes with $T_1$, i.e. $b_2(p_2, A_1) > 0$, then the optimal pricing strategy for $T_1$ is, let $b_2(p_2,a^{\ast,+}_1)=b_1(p,a_1^{\ast,+})$ and $b_1(p,a'_1)=0$,
\begin{equation}
    p_1=\underset{p}{argmax}\, p\int_{a'_1}^{a^{\ast,+}_1}\lambda(a)da.
\end{equation}

If there is no competition between the two models, then the scenario for $T_1$ is the same as \eqref{i1case}.

    \item {\bf In $T_2$'s perspective} For $T_2$, it makes choice to compete with $T_1$ or not. Therefore, we calculate the maximum values in two different scenarios separately to derive the optimal pricing strategy for $T_2$. If $T_2$ competes with $T_1$, we have, let $b_1(p_1,a_2^{\ast,-})=b_2(p,a_2^{\ast,-})$,
    \begin{equation}
        p_2=\underset{p}{argmax}\, p\int_{a^{\ast,-}_2}^{A_2}\lambda(a)da.
    \end{equation}
    If $T_2$ does not compete with $T_1$, i.e. $b_2(p,A_1)\leq0$, we have
    \begin{equation}
        p_2=\underset{p}{argmax}\, p\int_{a'_2}^{A_2}\lambda(a)da.
    \end{equation}
\end{itemize}
If there exist $p_1$ and $p_2$ such that they are respectively the optimal pricing strategies for $T_1$ and $T_2$, then a Nash equilibrium exists.

\subsection{General situation}
We are now considering the pricing strategy for a data chain of length $n$. The pricing strategy for $T_1$ is given, and the pricing strategy of $T_n$ is similar. For $T_i$, we consider the left and right endpoints of its market allocation interval, noting that the values of these endpoints essentially stem from the competitive game between $T_i$ and its neighboring models $T_{i-1}$ and $T_{i+1}$. We have
\begin{itemize}
    \item {\bf The competitive game between $T_{i-1}$ and $T_i$}  In this case, $T_i$ makes choice to compete with $T_{i-1}$ or not. Therefore, we calculate the maximum values in two different scenarios separately. If $T_i$ competes with $T_{i-1}$, we have, let $b_{i-1}(p_{i-1},a_{i}^{\ast,-})=b_{i}(p,a_{i}^{\ast,-})$,
    \begin{equation}
        p_i=\underset{p}{argmax}\,p\int_{a_{i}^{\ast,-}}^{\ast}\lambda(a)da
    \end{equation}
    If $T_i$ does not compete with $T_{i-1}$, i.e. $b_2(p_2,A_1)\leq0$, we have
    \begin{equation}
        p_i=\underset{p}{argmax}\,p\int_{a'_i}^{\ast}\lambda(a)da
    \end{equation}
    \item {\bf The competitive game between $T_i$ and $T_{i+1}$}  If $T_{i+1}$ competes with $T_i$, then the optimal pricing strategy for $T_i$ is, let $b_{i+1}(p_{i+1},a^{\ast,+}_{i})=b_i(p,a_i^{\ast,+})$, 
\begin{equation}
    p_i=\underset{p}{argmax}\, p\int_{\ast}^{a^{\ast,+}_i}\lambda(a)da
\end{equation}

If there is no competition between the two models, then we have
\begin{equation}
    p_i=\underset{p}{argmax}\, p\int_{\ast}^{A_i}\lambda(a)da
\end{equation}

\end{itemize}

If there exist $p_1,p_2,\cdots,p_n$ such that they are respectively the optimal pricing strategies for $T_i$'s, then a Nash equilibrium exists.

We observe that, regardless of adopting dynamic pricing strategies or static pricing strategies, we only consider the competitive game between adjacent models. This simplified form of the game problem aligns with real-world scenarios.

\begin{rem}
    In this section, we omit the proof of the connectivity of the market allocation for brevity. It is evident from the proof details of Theorem \ref{thm:connect} that we did not utilize the condition that $p_i$ is the optimal price under a static strategy. This implies that, regardless of the prices arbitrarily assigned, the final market allocation is always connected. Consequently, the connectivity of the market allocation also holds in this section.
\end{rem}

\section{Applications: separable utility function}
In this section, we discuss a special type of utility function, namely the separable utility function, and present the possible values for $p_i$'s in such a scenario. We have

\begin{dee}
   \cite{chen2022selling}If $b_i(p,a)=f_i(p)h_i(a)$ or $b_i(p,a)=f_i(p)+h_i(a)$ respectively, then $b_i$ is  multiplicatively separable, or additively separable respectively where $f_i$ is a function of $p$ and $h_i$ is a function of $a$.
\end{dee}

    Note that in our paper, the additive separable utility function and the multiplicative separable utility function can be transformed into each other through exponential or logarithmic transformations. Without loss of generality, we only consider the additive separable scenario. We assume
    $b_i(p,a)=f_i(p)+h_i(a)$ where $f_i$ and $h_i$ are differentiable.

Note that the revenue function $r_i(p)$ is continuous almost everywhere, implying that the maximum value is attained only when the price $p$ takes values at the endpoints or stationary points. The computation of revenue corresponding to the endpoints is straightforward. Moreover, the conclusion drawn in this section is that, under the utility functions being separable, the expression for the stationary points can be solved using differential equations. The goal of this section is to calculate all the expressions of the stationary points,
\subsection{Uniform distribution}
There is one model sold in the market, then we have
\begin{equation}
    p_1=\underset{p}{argmax}\, p\int_{a'_1}^{A_1}da=\underset{p}{argmax}\, p(A_1-a'_1).
\end{equation}
Now we calculate the stationary point and we have $b_1(p,a'_1)=f_1(p)+h_1(a'_1)=0$. It leads to $a'_1=h_1^{-1}(-f_1(p))$, for which $a'_1$ is a function of price $p$ and we denote $a'_1(\eta_1)=\xi_1$. We have
\begin{equation}
    (A_1-a'_1(p))-p\frac{d a'_1(p)}{dp}=0,
\end{equation}
and then $p$ is the solution of the following equations
\begin{equation}
\begin{aligned}
    a'_1=h_1^{-1}(-f_1(p)),\\
    a'_1=A_1+\frac{\eta_1(\xi_1-A_1)}{p}.
\end{aligned}
\end{equation}

In this case, the stationary point of $r_1$ is the maximum point.

In the general case, i.e. $i\geq2$, we have
\begin{itemize}
    \item For $a'_i\in(a_{k-1,+},a_{k,-}]$ where $a'_i=h_i^{-1}(-f_i(p))$, we have $b_i(f_i^{-1}(-h_i(a'_i)),a)>\mathfrak{b}_{i-1}(a)$ for $a\in(a'_i,A_{i-1}]$. We denote $a'_i(\eta_i)=\xi_i$ and we have
    \begin{equation}
        p_i=\underset{p}{argmax}\,p(A_i-a'_i),\quad
        a_{k-1,+}<a'_i\leq a_{k,-}.
\end{equation}
We calculate the stationary point and solve
\begin{equation}
    (A_i-a'_i)-p\frac{d a'_i(p)}{dp}=0.
\end{equation}
So the stationary point $p$ is the solution of the following equations
\begin{equation}
    \begin{aligned}
        a'_i=A_i+\frac{\eta_i(\xi_i-A_i)}{p},\\
        a'_i=h_i^{-1}(-f_i(p)).
    \end{aligned}
\end{equation}
Also, in the case of \eqref{aiai+1case}, we have $p_i=f_i^{-1}(h_i(A_j))$.
    \item Let $a'_i\in(a_{k,-},a_{k,+}]$ and $b_i(p,a_i^\ast)=\mathfrak{b}_{i-1}(a_i^\ast)=b_j(p_j,a_i^\ast)$ where $a_i^\ast\in(a_{j,-},a_{j,+}]$. In this case, we have
    \begin{equation}
        p_i=\underset{p}{argmax}\,p(A_i-a_i^\ast)+p_j(a_i^\ast-a_{j,-}),\quad
        a_{k,-}<a'_i\leq a_{k,+},\quad a_{j,-}<a_i^\ast\leq a_{j,+}
\end{equation}
where
    \begin{equation}
        a_i^\ast=(h_i-h_j)^{-1}(f_j(p_j)-f_i(p)).
    \end{equation}
    We solve, denote $a_i^\ast(\psi_i)=\phi_i$,
\begin{equation}
    A_i-a_i^\ast-(p-p_j)\frac{da_i^\ast}{dp}=0.
\end{equation} 
The stationary point $p$ is the solution of the following equations
\begin{equation}
    \begin{aligned}
        a_i^\ast=A_i+\frac{\phi_i-A_i}{(\psi_i-p_j)(p-p_j)},\\
        a_i^\ast=(h_i-h_j)^{-1}(f_j(p_j)-f_i(p)).
    \end{aligned}
\end{equation}
    \item If $a'_i\in(A_{i-1},A_i]$, then we have
\begin{equation}
        p_i=\underset{p}{argmax}\,p(A_i-a'_i)da,\quad
        A_{i-1}<a'_i\leq A_i,
\end{equation}
and
\begin{equation}
    (A_i-a'_i(p))-p\frac{d a'_i(p)}{dp}=0.
\end{equation}
It leads to $p$ is the solution of the following equations
\begin{equation}
\begin{aligned}
    a'_i=h_i^{-1}(-f_i(p)),\\
    a'_i=A_i+\frac{\eta_i(\xi_i-A_i)}{p}.
\end{aligned}
\end{equation}
\end{itemize}

\subsection{Continuous distribution}
We define the cumulative distribution function (CDF) of $\lambda(a)$ as $F(a)$. There is one model sold in the market, then we have
\begin{equation}
    p_1=\underset{p}{argmax} p\int_{a'_1}^{A_1}\lambda(a)da=\underset{p}{argmax} p(F(A_1)-F(a'_1))
\end{equation}
Now we calculate the stationary point, where it is also the maximum point. We have
\begin{equation}
    (F(A_1)-F(a'_1))-p\frac{d F(a'_1)}{dp}=0.
\end{equation}
then $p$ is the solution of the following equations
\begin{equation}
\begin{aligned}
    a'_1=h_1^{-1}(-f_1(p)),\\
    a'_1=F^{-1}(F(A_1)+\frac{\eta_1(\xi_1-A_1)}{p}).
\end{aligned}
\end{equation}

In the general case, i.e. $i\geq2$, we have
\begin{itemize}
    \item If $a'_i\in(a_{k-1,+},a_{k,-}]$ and $b_i(f_i^{-1}(-h_i(a'_i)),a)>\mathfrak{b}_{i-1}(a)$ in $(a'_i,A_{i-1}]$, then, similarly, we calculate the stationary point of \eqref{geneicase1} and solve
\begin{equation}
    (F(A_i)-F(a'_i))-p\frac{d F(a'_i)}{dp}=0.
\end{equation}
It leads to
\begin{equation}
\begin{aligned}
    a'_i=h_i^{-1}(-f_i(p)),\\
    a'_i=F^{-1}(F(A_i)+\frac{\eta_i(\xi_i-A_i)}{p}).
\end{aligned}
\end{equation}
Also, in the case of \eqref{aiai+1case}, we have $p_i=f_i^{-1}(h_i(A_j))$.
    \item If $a'_i\in(a_{k,-},a_{k,+}]$ and $b_i(p,a_i^\ast)=\mathfrak{b}_{i-1}(a_i^\ast)$ where $a_i^\ast\in(a_{j,-},a_{j,+}]$, then we have 
    \begin{equation}
        a_i^\ast=(h_i-h_j)^{-1}(f_j(p_j)-f_i(p)).
    \end{equation}
    We calculate the stationary point of \eqref{geneicase2}, and 
\begin{equation}
    \begin{aligned}
        a_i^\ast=F^{-1}(F(A_i)+\frac{\phi_i-A_i}{(\psi_i-p_j)(p-p_j)}),\\
        a_i^\ast=(h_i-h_j)^{-1}(f_j(p_j)-f_i(p)).
    \end{aligned}
\end{equation}
    
    \item  
If $a'_i\in(A_{i-1},A_i]$, then we calculate the stationary point of \eqref{geneicase4} and
\begin{equation}
\begin{aligned}
    a'_i=h_i^{-1}(-f_i(p)),\\
    a'_i=F^{-1}(F(A_i)+\frac{\eta_i(\xi_i-A_i)}{p}).
\end{aligned}
\end{equation}

\end{itemize}

\section{Incomplete information game}\label{sec:incinf}
We consider incomplete information games, where a discrepancy exists between the seller's known utility functions and the true utility functions of the buyers. The goal of this section is to demonstrate that our method remains robust under such errors and provide an error analysis for the revenue of a single model. We recall the notation in Section \ref{sec:stathote}, and $b_i$'s are the utility functions known to the seller. We denote $\mathbf{b}_i$'s are the true utility functions in the market corresponding to the cost 
$c_i$, and $\Lambda=\sup\,\lambda(a)$. We present the following assumptions:
\begin{itemize}
    \item
    For the utility functions known to the seller, we assume
    \begin{equation}
    \alpha_{i,p}|a-a'|\leq|b_i(p,a)-b_i(p,a')|\leq \beta_{i,p}|a-a'|
\end{equation}
where $0<\alpha_1<\beta_1<\alpha_2<\beta_2<\cdots<\alpha_n<\beta_n$. We use a slight abuse of notation and denote $\alpha_{i,p}=\alpha_i$ and $\beta_{i,p}=\beta_i$ if it does not introduce ambiguity.
\item 
The true utility function $\mathbf{b}_i$ are also continuous and accuracy-compatible. For the errors between the utility functions known to the seller and the true functions, we assume
\begin{equation}
    |\mathbf{b}_i(p,a)-b_i(p,a)|<\epsilon_i(p).
\end{equation}
\item We assume, for $i<j$ and $p<p'$,
\begin{equation}
    0>\underset{a\to\underline{A}}{\lim}{b_i(p,a)}>\underset{a\to\underline{A}}{\lim}{b_j(p',a)},\quad 0>\underset{a\to\underline{A}}{\lim}{\mathbf{b}_i(p,a)}>\underset{a\to\underline{A}}{\lim}{\mathbf{b}_i(p',a)},
\end{equation}
\end{itemize}

We recall Section \ref{sec:stathote} and the corresponding optimal revenue of $T_i$ is
\begin{equation}
    r_i=p_i\int_{a_{i,-}}^{a_{i,+}}\lambda(a)da
\end{equation}
where $r_i-c_i>0$. Note that in this case, we have $a_{i,-}\in\{a'_i,a_i^\ast\}$ and $a_{i,+}\in\{A_i,a_{i+1}^\ast\}$ if $a_i^\ast$ and $a_{i+1}^\ast$ are attainable.

Now we calculate the true revenue. We denote $\mathbf{b}_i(p_i,\mathbf{a}'_i)=0$ and $\mathbf{b}_j(p_j,\mathbf{a}_j^\ast)=\mathbf{b}_i(p_i,\mathbf{a}_j^\ast)>0$ if $\mathbf{a}_j^\ast$ is attainable, then
\begin{equation}
    \mathbf{r}_i=\max\{p_i\int_{\mathbf{a}_{i,-}}^{\mathbf{a}_{i,+}}\lambda(a)da,0\}
\end{equation}
where
\begin{equation}
    \mathbf{a}_{i,-}=\max\{\mathbf{a}'_i,\mathbf{a}_j^\ast,j=1,2,\ldots,i-1\},\quad \mathbf{a}_{i,+}=\min\{A_i,\mathbf{a}_j^\ast,j=i+1,i+2,\ldots,n\}.
\end{equation}
It leads to
\begin{equation}
    |r_i-\mathbf{r}_i|\leq p_i\Lambda(|a_{i,-}-\mathbf{a}_{i,-}|+|a_{i,+}-\mathbf{a}_{i,+}|).
\end{equation}

We calculate the error bound through classified discussion.
\begin{itemize}
    \item If $a_{i,-}=a'_i$ and $\mathbf{a}_{i,-}=\mathbf{a}'_i$, then we have
    \begin{equation}
        |a'_i-\mathbf{a}'_i|\leq\frac{1}{\alpha_i}|b_i(p_i,a'_i)-b_i(p_i,\mathbf{a}'_i)|=\frac{1}{\alpha_i}|\mathbf{b}_i(p_i,\mathbf{a}'_i)-b_i(p_i,\mathbf{a}'_i)|<\frac{\epsilon_i(p_i)}{\alpha_i}.
    \end{equation}
    \item If $a_{i,-}=a_i^\ast$ and $\mathbf{a}_{i,-}=\mathbf{a}_{i-1}^\ast$, then
we recall $B_{i,i-1}(a)=b_i(p_i,a)-b_{i-1}(p_{i-1},a)$ and let $\mathbf{B}_{i,i-1}(a)=\mathbf{b}_i(p_i,a)-\mathbf{b}_{i-1}(p_{i-1},a)$. So $B_{i,i-1}(a_i^\ast)=\mathbf{B}_{i,i-1}(\mathbf{a}_{i-1}^\ast)=0$. We have
\begin{equation}
    |B_{i,i-1}(\mathbf{a}_i^\ast)-\mathbf{B}_{i,i-1}(\mathbf{a}_i^\ast)|\leq|b_i(p_i,\mathbf{a}_i^\ast)-\mathbf{b}_i(p_i,\mathbf{a}_i^\ast)|+|b_j(p_j,\mathbf{a}_i^\ast)-\mathbf{b}_j(p_j,\mathbf{a}_i^\ast)|<\epsilon_i(p_i)+\epsilon_j(p_j)
\end{equation}
and
\begin{multline}
    |B_{i,i-1}(\mathbf{a}_{i-1}^\ast)-\mathbf{B}_{i,i-1}(\mathbf{a}_{i-1}^\ast)|=|B_{i,i-1}(\mathbf{a}_{i-1}^\ast)-B_{i,i-1}(a_i^\ast)|\\
    =|b_i(p_i,\mathbf{a}_i^\ast)-b_{i-1}(p_{i-1},\mathbf{a}_i^\ast)-b_i(p_i,a_i^\ast)+b_{i-1}(p_{i-1},a_i^\ast)|\\ 
    \geq ||b_i(p_i,\mathbf{a}_i^\ast)-b_i(p_i,a_i^\ast)|-|b_{i-1}(p_{i-1},\mathbf{a}_i^\ast)-b_{i-1}(p_{i-1},a_i^\ast)||\geq(\alpha_i-\beta_{i-1})|\mathbf{a}_{i-1}^\ast-a_{i-1}^\ast|
\end{multline}
It leads to
\begin{equation}\label{Berrboun}
    |a_i^\ast-\mathbf{a}_{i-1}^\ast|<\frac{\epsilon_i(p_i)+\epsilon_j(p_j)}{\alpha_i-\beta_{i-1}}
\end{equation}

\item If $a_{i,-}=a'_i$ and $\mathbf{a}_{i,-}=\mathbf{a}_j^\ast$ where $j\neq i-1$, then $a'_i\geq A_j\geq \mathbf{a}_j^\ast$ and $\mathbf{a}_j^\ast\geq\mathbf{a}'_i$ due to the the monotonicity of $\mathbf{b}_j$. So $|a'_i-\mathbf{a}_j^\ast|\leq|a'_i-\mathbf{a}'_i|<\frac{\epsilon_i(p_i)}{\alpha_i}$.

\item If $a_{i,-}=a_i^{\ast}$, $\mathbf{a}_{i,-}=\mathbf{a}_j^{\ast}$ where $j\neq i-1$ and $b_i(p_i,a_j^\ast)=b_j(p_j,a_j^\ast)$ is attainable, then we have $a_i^\ast>a_j^\ast$ and $\mathbf{a}_j^\ast\geq\mathbf{a}_{i-1}^\ast$. It leads to
\begin{equation}\label{Bijerrboun}
    |a_i^\ast-\mathbf{a}_j^\ast|\leq \max\{|a_i^\ast-\mathbf{a}_{i-1}^\ast|,|a_j^\ast-\mathbf{a}_j^\ast|\}\leq\max\{\frac{\epsilon_i(p_i)+\epsilon_{i-1}(p_{i-1})}{\alpha_i-\beta_{i-1}},\frac{\epsilon_i(p_i)+\epsilon_j(p_j)}{\alpha_i-\beta_j}\}.
\end{equation}
If $a_j^\ast$ is not attainable, then $a_i^\ast>A_j>\mathbf{a}_j^\ast\geq \mathbf{a}_{i-1}^\ast$. It leads to
\begin{equation}
    |a_i^\ast-\mathbf{a}_j^\ast|< |a_i^\ast-\mathbf{a}_{i-1}^\ast| <\frac{\epsilon_i(p_i)+\epsilon_{i-1}(p_{i-1})}{\alpha_i-\beta_{i-1}}
\end{equation}
\item If $a_{i,-}=a_i^\ast$ and $\mathbf{a}=\mathbf{a}'_i$, then we have, if  $\mathbf{a}'_i\geq A_{i-1}$, then $\mathbf{a}'_i\geq A_{i-1}\geq a_i^\ast>a'_i$ and
\begin{equation}
    |a_i^\ast-\mathbf{a}'_i|<|a'_i-\mathbf{a}'_i|<\frac{\epsilon_i(p_i)}{\alpha_i}.
\end{equation}
If $\mathbf{a}'_i< A_{i-1}$, then $\mathbf{a}_i^\ast$ is attainable due to $\mathbf{b}_{i-1}(p_{i-1},A_{i-1})\leq\mathbf{b}_i(p_i,A_{i-1})$.
It leads to
\begin{equation}
    |\mathbf{a}'_i-a_i^\ast|\leq \max\{|a_i^\ast-\mathbf{a}_{i-1}^\ast|,|a'_i-\mathbf{a}'_i|\}\leq\max\{\frac{\epsilon_i(p_i)+\epsilon_{i-1}(p_{i-1})}{\alpha_i-\beta_{i-1}},\frac{\epsilon_i(p_i)}{\alpha_i}\}.
\end{equation}
    \item If $a_{i,+}=a_{i+1}^\ast$ and $\mathbf{a}_{i,+}=\mathbf{a}_{i+1}^\ast$, then we have, similar to \eqref{Berrboun},
    \begin{equation}
    |a_i^\ast-\mathbf{a}_{i+1}^\ast|<\frac{\epsilon_i(p_i)+\epsilon_{i+1}(p_{i+1})}{\beta_{i+1}-\alpha_i}
\end{equation}
    \item If $a_{i,+}=a_{i+1}^\ast$ and $\mathbf{a}_{i,+}=\mathbf{a}_j^\ast$ where $j\neq i+1$, then similar to \eqref{Bijerrboun}, we have
    \begin{equation}
    |a_i^\ast-\mathbf{a}_j^\ast|\leq \max\{|a_i^\ast-\mathbf{a}_{i+1}^\ast|,|a_j^\ast-\mathbf{a}_j^\ast|\}\leq\max\{\frac{\epsilon_i(p_i)+\epsilon_{i+1}(p_{i+1})}{\beta_{i+1}-\alpha_i},\frac{\epsilon_i(p_i)+\epsilon_j(p_j)}{\beta_j-\alpha_i}\}
\end{equation}
if $a_j^\ast$ is attainable, and
\begin{equation}
    |a_i^\ast-\mathbf{a}_j^\ast|< |a_i^\ast-\mathbf{a}_{i-1}^\ast| <\frac{\epsilon_i(p_i)+\epsilon_{i+1}(p_{i+1})}{\beta_{i+1}-\alpha_i}
\end{equation}
if $a_j^\ast$ is not attainable.
    \item If $a_{i,+}=A_i$ and $\mathbf{a}_{i,+}=\mathbf{a}_j^\ast$, then we have, note that $B_{i,j}(A_i)>0$, $\mathbf{B}_{i,j}(A_i)\leq 0$, and $B_{i,j}(\mathbf{a}_j^\ast)>0$,
    \begin{equation}
    |B_{i,j}(A_i)-\mathbf{B}_{i,j}(A_i)|\geq|B_{i,j}(A_i)-B_{i,j}(\mathbf{a}_i^\ast)|\geq(\beta_{j}-\alpha_i)|A_i-\mathbf{a}_i^\ast|
\end{equation}
It leads to
\begin{equation}
    |A_i-\mathbf{a}_i^\ast|\leq\frac{\epsilon_i(p_i)+\epsilon_{i+1}(p_{j})}{\beta_{j}-\alpha_i}. 
\end{equation}
    \item If $a=a_i^\ast$ and $\mathbf{a}=A_i$, then we have, note that $B_{i.i+1}(A_i)<0$, $\mathbf{B}_{i,i+1}(A_i)>0$, and $B_{i,i+1}(a_i^\ast)=0$,
    \begin{equation}
    |B_{i,i+1}(A_i)-\mathbf{B}_{i,i+1}(A_i)|\geq|B_{i,i+1}(A_i)-B_{i,i+1}(a_i^\ast)|\geq(\beta_{i+1}-\alpha_i)|A_i-a_i^\ast|
\end{equation}
It leads to
\begin{equation}
    |A_i-a_i^\ast|\leq\frac{\epsilon_i(p_i)+\epsilon_{i+1}(p_{i+1})}{\beta_{i+1}-\alpha_i}. 
\end{equation}
\end{itemize}

\section{Conclusions}
We study to trade models instead of datasets themselves in this paper, and we have introduced a novel transaction mechanism in the data market. Given the reproducibility of data, we train the data multiple times to produce models of varying accuracy levels, catering to diverse market demands. We employ the generalized Hotelling's law to describe the market preferences for these models and the competitive scenario. Based on this scenario, we select the most appropriate cost allocation strategy to ensure maximum profit maximization.


%
%
%
%
%

\bibliographystyle{ACM-Reference-Format}
\bibliography{sample-bibliography}

\appendix

\end{document}